\pdfoutput=1
\documentclass{article}
\usepackage[T1]{fontenc}

\usepackage[accepted]{icml2023}
\newcommand{\doi}[1]{\href{https://doi.org/#1}{{doi: \urlstyle{rm}\nolinkurl{#1}}}}

\usepackage{tikz}
\usepackage{tikz-qtree}
\usetikzlibrary{positioning,arrows.meta,calc,decorations.pathreplacing,calligraphy}
\usepackage{pgfplots}
\pgfplotsset{compat=1.16}
\input{nn}

\usepackage{microtype}
\usepackage{booktabs}
\usepackage{enumitem}
\usepackage[bibliography=common]{apxproof}
\usepackage{amsmath}
\usepackage{mathtools}

\usepackage[hidelinks]{hyperref}
\usepackage{amsthm}
\usepackage[capitalize,nameinlink]{cleveref} 

\crefformat{section}{#2\S #1#3}
\Crefformat{section}{#2Section~#1#3}
\crefformat{subsection}{#2\S #1#3}
\Crefformat{subsection}{#2Section~#1#3}
\crefformat{subsubsection}{#2\S #1#3}

\newtheorem{theorem}{Theorem}
\newtheorem{proposition}[theorem]{Proposition}
\newtheorem{lemma}[theorem]{Lemma}

\theoremstyle{definition}
\newtheorem{definition}{Definition}

\usepackage{amssymb,stmaryrd}
\usepackage[scaled=1.06]{inconsolata}
\usepackage[scaled=0.9]{helvet}
\usepackage{newtxmath} 

\usepackage{natbib}
\newcommand{\R}{\ensuremath{\mathbb{R}}}
\newcommand{\N}{\ensuremath{\mathbb{N}}}

\newcommand{\indicator}[1]{\mathopen{\mathbb{I}}\left[#1\right]}
\newcommand{\floor}[1]{\left\lfloor #1 \right\rfloor}
\newcommand{\bit}[2]{\left\langle#2\right\rangle_{#1}}
\newcommand{\bitneg}{\mathbin\sim}
\renewcommand{\cong}[1]{\mathbin{\equiv_{#1}}}

\newcommand{\fixednums}{\ensuremath{\mathbb{F}}}
\newcommand{\extranums}{\ensuremath{\mathbb{E}}}
\newcommand{\intbits}{r}
\newcommand{\fracbits}{s}
\newcommand{\extradigit}{\text{extra}}
\newcommand{\extra}[1]{\bit{\extradigit}{#1}}
\newcommand{\fixed}[1]{\bit{\text{fixed}}{#1}}

\newcommand{\alphabet}{\Sigma}
\newcommand{\letter}{a}
\newcommand{\cls}{\texttt{CLS}}

\newcommand{\selfatt}{\ensuremath{\text{SA}}}
\newcommand{\ffnn}{\ensuremath{\text{FF}}}
\newcommand{\dffnn}{d_\text{FF}}
\newcommand{\dk}{d_\text{K}}
\newcommand{\layernorm}{\ensuremath{\text{LN}}}
\newcommand{\stack}{\ensuremath{\text{Stack}}}
\newcommand{\encoder}{\ensuremath{\text{Enc}}}
\newcommand{\layer}{\ensuremath{\operatorname{Layer}}}
\newcommand{\wq}{{W^\text{(Q)}}}
\newcommand{\wk}{{W^\text{(K)}}}
\newcommand{\wv}{{W^\text{(V)}}}
\newcommand{\wone}{{W^{(1)}}}
\newcommand{\wtwo}{{W^{(2)}}}
\newcommand{\bone}{{b^{(1)}}}
\newcommand{\btwo}{{b^{(2)}}}

\newcommand{\logic}{\ensuremath{\mathsf{FOC}[\mathord+;\mathsf{MOD}]}}
\newcommand{\isletter}[2]{Q_{#1}(#2)}
\newcommand{\congruent}[3]{\mathsf{MOD}^{#1}_{#2}(#3)}
\newcommand{\counteq}[1]{\exists^{{}=#1}}
\newcommand{\sentence}{\sigma}
\newcommand{\formula}{\phi}
\newcommand{\Chi}{X}
\newcommand{\true}{\top}
\newcommand{\false}{\bot}

\newcommand{\tczero}{\ensuremath{\mathsf{TC}^0}}

\begin{document}

\twocolumn[
  \icmltitle{Tighter Bounds on the Expressivity of Transformer Encoders}
  \begin{icmlauthorlist}
    \icmlauthor{David Chiang}{nd}
    \icmlauthor{Peter Cholak}{nd}
    \icmlauthor{Anand Pillay}{nd}
  \end{icmlauthorlist}
  \icmlaffiliation{nd}{University of Notre Dame, USA}
  \icmlcorrespondingauthor{David Chiang}{dchiang@nd.edu}
  \icmlkeywords{Keywords}
  \vskip0.3in
]
\printAffiliationsAndNotice{}

\begin{abstract}
Characterizing neural networks in terms of better-understood formal systems has the potential to yield new insights into the power and limitations of these networks.
Doing so for transformers remains an active area of research.
Bhattamishra and others have shown that transformer encoders are at least as expressive as a certain kind of counter machine, while Merrill and Sabharwal have shown that fixed-precision transformer encoders recognize only languages in
uniform $\tczero$.
We connect and strengthen these results by identifying a variant of first-order logic with counting quantifiers that is simultaneously an upper bound for fixed-precision transformer encoders and a lower bound for transformer encoders. This brings us much closer than before to an exact characterization of the languages that transformer encoders recognize.
\end{abstract}

\section{Introduction}

Characterizing neural networks in terms of better-understood formal systems has the potential to yield new insights into the power and limitations of these networks.
Recurrent neural networks (RNNs) were linked to finite automata from the start \citep{mcculloch+pitts:1943,kleene:1956} and have continued to be studied using finite automata (see, e.g., the survey by \citet{forcada+carrasco:2001}).
Convolutional neural networks, too, have been related to finite automata \citep{schwartz+:2018}.

Transformers \citep{vaswani+:2017} have been studied in relation to counter machines \citep{bhattamishra+:2020}, Boolean circuits \citep{hao-etal-2022-formal,merrill-etal-2022-saturated,merrill+sabharwal:2023log}, and programming languages \citep{weiss+:2021}, obtaining various upper and lower bounds on their expressivity. (\Cref{sec:related} gives a more detailed survey.)
As a lower bound, \citet{bhattamishra+:2020} show that transformer encoders are at least as powerful as \emph{simplified stateless counter machines} (SSCMs), which test whether the numbers of occurrences of input symbols satisfy a given linear constraint.
As an upper bound, \citet{merrill+sabharwal:2023} restrict to \emph{fixed-precision} transformer encoders and show that they are in uniform $\tczero$.

\begin{figure}[t]
\centering
\tikzset{every node/.style={align=center,execute at begin node={\small}}}
\tikzset{strict/.style={line width=0.4mm,->,>=Latex}}
\tikzset{unknown/.style={line width=0.2mm,->,>={Latex[open,scale=1.5]}}}
\begin{tikzpicture}[x=3cm,y=3cm]
\node(logic) at (0,0) {\logic};
\node(te) at (0,1) {transformer encoders};
\node(fixedte) at (0,-1) {fixed-precision \\ transformer encoders};
\node(sscm) at (1,-0.5) {simplified stateless \\ counter machines};
\node(tc0) at (-1,0.5) {uniform \tczero};
\tikzset{every node/.style={execute at begin node={\tiny}}}
\draw[unknown] (fixedte) to node[sloped,auto=left] {\cref{thm:upper}} (logic);
\draw[unknown] (logic) to node[sloped,auto=left] {\cref{thm:lower}} (te);
\draw[strict] (sscm) to node[sloped,auto=left] {Prop.~\labelcref{thm:sscm}} (logic);
\draw[strict,bend left=10] (sscm) to node[text width=1.5cm,align=left,auto=right] {\citealp{bhattamishra+:2020}} (te);
\draw[strict] (logic) to node[sloped,auto=left] {Prop.~\labelcref{thm:tc0}} (tc0);
\draw[strict,bend right=10] (fixedte) to node[text width=1.5cm,align=left,auto=left] {\citealp{merrill+sabharwal:2023}} (tc0);
\end{tikzpicture}
  \caption{Overview of results. Arrows indicate inclusion, and thick arrows indicate strict inclusion. We show that \logic{} is simultaneously a tighter upper bound on fixed-precision transformer encoders than uniform \tczero{} is, and a tighter lower bound on transformer encoders than SSCMs are.}
  \label{fig:overview}
\end{figure}
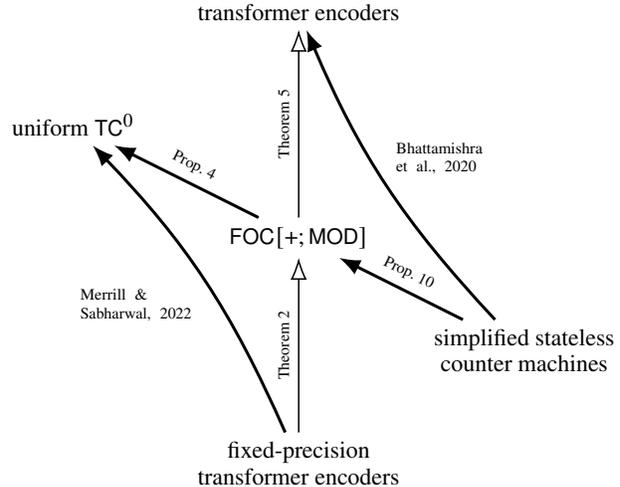

Here, we study transformers in relation to first-order logic. (\citet{merrill+sabharwal:2023} also relate them to first-order logic, but indirectly through circuits.)
We define \logic, which is first-order logic with counting quantifiers, where positions have modular predicates but not ordering, and counts have ordering and addition (\cref{sec:logic}). We then connect and strengthen the two above-mentioned results by showing that \logic{} is simultaneously an upper bound for fixed-precision transformer encoders (\cref{sec:upper}) and a lower bound for transformer encoders (\cref{sec:lower}).

These two results counterbalance each other.
As an upper bound on fixed-precision transformer encoders, \logic{} might be imagined to define languages far beyond the power of transformers,
but the lower bound assures us that it does not have any expressivity that does anything ``un-transformer-like.''
As a lower bound on (arbitrary-precision) transformer encoders, \logic{} might be imagined to be far weaker than transformers,
but the upper bound assures us that it can express everything that real-world (fixed-precision) transformers can.
Together, these two results bring us much closer than before to an exact characterization of the languages that transformer encoders recognize.

\section{Preliminaries}

We write
$\N$ for the set of natural numbers, which includes 0.
If $a$ and $b$ are integers, we write $[a,b]$ for the set $\{x \in \mathbb{Z} \mid a \leq x \leq b\}$, and
$a \cong{m} b$ iff $a$ and $b$ are congruent modulo~$m$.

We make frequent use of the Iverson bracket $\indicator{\phi}$, which has the value $1$ if the statement $\phi$ is true, and $0$ otherwise.

If $M$ is a matrix, we write $M_{i,*}$ for the $i$-th row of $M$ and $M_{*,j}$ for the $j$-th column of $M$. We write $\mathbf{0}$ for the zero vector or matrix, $\mathbf{0}^n$ for the $n$-dimensional zero vector, and $\mathbf{I}$ for the identity matrix.

We often work with families of functions $f = \bigl(f^{(n)}\bigr)_{n>0}$ where each $f^{(n)} \colon X^n \rightarrow Y^n$. For brevity, we write $f \colon X^n \rightarrow Y^n$ and apply $f$ to strings or vectors of any length.

\section{Transformers}
\label{sec:transformers}

A transformer can have an encoder and/or a decoder; following previous work \citep{bhattamishra+:2020,hahn:2020,hao-etal-2022-formal,merrill-etal-2022-saturated,merrill+sabharwal:2023log}, we focus on transformer encoders.

The input is a string $w_1 \cdots w_n$, to which we prepend a special symbol \cls{} at position 0.
Thus the network sees a sequence of $n+1$ symbols; to avoid clutter, we write $n' = n+1$.
The string is converted to an \emph{activation} matrix $A^{(0)} \in \R^{d \times n'}$, where column $A^{(0)}_{*,p}$ represents symbol $w_p$, and $d$ is the \emph{width} of the network.
A stack of $L$ layers maps $A^{(0)}$ to $A^{(1)}, A^{(2)}, \ldots, A^{(L)} \in \R^{d \times n'}$.
Then we apply a sigmoid layer to the output at \cls{} (that is, $A^{(L)}_{*,0}$) to obtain a single probability. The network, which we call a \emph{transformer classifier}, accepts $w$ iff this probability is at least $\tfrac12$.

The rest of this section describes each of these components in more detail. Readers already familiar with transformers may safely skip ahead.

\subsection{Input layer}

Each input vector $A^{(0)}_{*,p}$ is the sum of a \emph{word embedding} and a \emph{positional encoding}.
\begin{definition}
  A \emph{word embedding} on $\alphabet$ with width~$d$ is a mapping $\text{WE} \colon \alphabet \rightarrow \R^{d}$ from symbols to vectors.
\end{definition}

\begin{definition}
  A \emph{sinusoidal positional encoding} with (even) width~$d$ is a mapping from positions to vectors,
  \begin{align*}
    \text{PE} \colon \mathbb{N} &\rightarrow \R^{d} \\
        p &\mapsto
        \begin{bmatrix}
          \sin 2\pi\xi_1 p \\
          \cos 2\pi\xi_1 p \\
          \vdots \\
          \sin 2\pi\xi_{d/2} p \\
          \cos 2\pi\xi_{d/2} p
        \end{bmatrix} \qquad \xi_i \in \mathbb{Q}
\end{align*}
\end{definition}
In the original paper, the $\xi_i$ are set to fixed values ($\xi_i = \frac{i-1}{2\pi}$), which are not rational. Here, we assume that they are rational (needed for \cref{thm:upper}), and can be set to arbitrary rationals (needed for \cref{thm:lower}). Since the rationals are dense in the reals, we can choose them to be as close to the original values as we want.

\subsection{Hidden layers}

Each hidden layer has a \emph{self-attention} followed by a \emph{position-wise feed-forward network}.
\begin{definition}
A \emph{self-attention sublayer} with width $d$ and key width $\dk$ is a function
  \begin{align}
    \selfatt \colon \R^{d \times n'} &\rightarrow \R^{d \times n'} \notag \\
    A &\mapsto \begin{bmatrix} c_0 & \cdots & c_n \end{bmatrix} \ \text{where} \notag \\
    s_{qp} &= \frac{\wq A_{*,q} \cdot \wk A_{*,p}}{\sqrt d} \label{eq:logits} \\
    c_q &= \frac{\sum_{p=0}^n (\exp s_{qp}) \wv A_{*,p} }{\sum_{p=0}^n \exp s_{qp}} \label{eq:context}
  \end{align}
and $\wq, \wk \in \R^{d \times \dk}$ and $\wv \in \R^{d \times d}$ are learned.
\end{definition}

\begin{definition}
  A \emph{position-wise feed-forward network} (FFN) with width $d$ and hidden width $\dffnn$ is a function
  \begin{align*}
    \ffnn \colon \R^d &\rightarrow \R^d \\
    x &\mapsto \wtwo \left( \max(0, \wone x + \bone) \right) + \btwo
  \end{align*}
  where the $\max(0, \mathord-)$, called a \emph{ReLU}, is taken elementwise, and
  $\wone \in \R^{\dffnn \times d}$, $\bone \in \R^{\dffnn}$, $\wtwo \in \R^{d \times \dffnn}$, $\btwo \in \R^d$
  are learned.

  We also apply $f$ column-wise to matrices $A \in \R^{d \times n'}$:
  \begin{equation}
    \ffnn(A) = \begin{bmatrix} \ffnn(A_{*,0}) & \cdots & \ffnn(A_{*,n}) \end{bmatrix}. \label{eq:vectorize}
  \end{equation}
\end{definition}

\begin{definition} \label{def:layer}
  A \emph{transformer layer} with $H$ heads and width~$d$ is a function
  \begin{align*}
    \layer \colon \R^{d \times n'} &\rightarrow \R^{d \times n'} \\
    A &\mapsto A'' \ \text{where} \\
    A' &= \sum_{h=1}^H \selfatt^{(h)}(A) + A \\
    A'' &= \ffnn(A') + A'
  \end{align*}
  where the $\selfatt^{(h)}$ are self-attentions, and $\ffnn$ is a position-wise FFN. The ${}+A$ and ${}+A'$ terms are known as \emph{residual connections}.
\end{definition}

We have omitted layer normalization \citep{ba+:2016} to simplify proofs. \Cref{sec:layer_norm} explains how to add layer normalization to our definitions and proofs.

\subsection{Stacks, encoders and classifiers}

\begin{definition}
  A \emph{transformer stack} with width $d$ is a function
  \begin{align*}
    \stack &\colon \R^{d \times n'} \rightarrow \R^{d \times n'} \\
    \stack &= \layer^{(L)} \circ \dots \circ \layer^{(1)}
  \end{align*}
  where each $\layer^{(\ell)}$ is a transformer layer with width $d$.
\end{definition}

\begin{definition}
A \emph{transformer encoder} with width $d$ is a function from strings to sequences of vectors,
\begin{align*}
  \encoder \colon \alphabet^n &\rightarrow \R^{d \times n'} \\
  w &\mapsto \stack(A) \ \text{where} \\
  A &= \begin{bmatrix} \text{WE}(\cls) & \text{WE}(w_1) & \cdots & \text{WE}(w_n) \end{bmatrix} \\
  &\qquad + \begin{bmatrix} \text{PE}(0) & \text{PE}(1) & \cdots & \text{PE}(n) \end{bmatrix}
\end{align*}
where $\text{WE}$ is a word embedding, $\text{PE}$ is a positional encoding, and $\stack$ is a transformer stack with width $d$.
\end{definition}

\begin{definition}
  A \emph{transformer classifier} with width $d$ is a function
  \begin{align}
    \text{Cls} \colon \alphabet^* &\rightarrow \R \notag \\
    w &\mapsto \text{sigmoid}\left(W \left[\text{Enc}(w)\right]_{*,0} + b\right) \label{eq:output_layer}
  \end{align}
  where $\text{Enc}$ is a transformer encoder with width $d$.
  We say that $\text{Cls}$ \emph{accepts} $w$ if $\text{Cls}(w) \geq \tfrac12$, and the language recognized by $\text{Cls}$ is $\{ w \in \alphabet^* \mid \text{$\text{Cls}$ accepts $w$}\}$.  
\end{definition}  

\section{First-Order Logic with Counting Quantifiers}
\label{sec:logic}

Instead of characterizing problems (formal languages) using devices for producing or consuming strings, we can characterize them using \emph{logical formulas} that declare what properties a string must have. The classic result in this approach is that the languages of finite strings described by monadic second-order logic are exactly those recognized by finite automata \citep{buchi:1960}. \citet{merrill+sabharwal:2023} relate the expressivity of transformers to that of first-order logic with majority quantifiers, but only indirectly via circuits. Here, we relate the expressivity of transformers directly to a logic called \logic.

\subsection{Examples}

We begin with a few examples of sentences of \logic{} that define languages. Assume $\alphabet = \{\texttt{0}, \texttt{1}\}$.
\begin{enumerate}
\item $\forall p. \isletter{\texttt{0}}{p} \lor \forall p. \isletter{\texttt{1}}{p}$ defines the language $\texttt{0}^* \cup \texttt{1}^*$. The variable $p$ ranges over positions of $w$, and for any symbol $\letter \in \alphabet$, $\isletter{\letter}{p}$ is true iff the symbol at position $p$ is $\letter$. So this sentence says that all symbols are $\texttt{0}$ or all are $\texttt{1}$.
\item $\forall p. (\congruent{0}{2}{p} \rightarrow \isletter{\texttt{0}}{p} \land \congruent{1}{2}{p} \rightarrow \isletter{\texttt{1}}{p})$ defines the language $(\texttt{10})^* \cup (\texttt{10})^*\texttt{1}$. For any $r \ge 0, m > 0$, the predicate $\congruent{r}{m}{p}$ tests whether $p \cong{m} r$. So this sentence says that all symbols in even positions are $\texttt{0}$'s and those in odd positions are $\texttt{1}$'s.
\item\label{item:equal} $\exists x. \left( \counteq{x}p. \isletter{\texttt{0}}{p} \land \counteq{x}p. \isletter{\texttt{1}}{p} \right)$ defines the language of strings with an equal number of $\texttt{0}$'s and $\texttt{1}$'s. The variable $x$ ranges over numbers. The $\counteq{x}$ is a \emph{counting quantifier}; the subformula $\counteq{x}p.\isletter{\texttt{0}}{p}$ says that there are exactly $x$ positions $p$ that make $\isletter{\texttt{0}}{p}$ true. Similarly, $\counteq{x}p.\isletter{\texttt{1}}{p}$ says that there are exactly $x$ occurrences of $\texttt{1}$.
  \item $\exists x. \exists y. \left( \counteq{x}p. \isletter{\texttt{0}}{p} \land \counteq{y}p. \isletter{\texttt{1}}{p} \land 2x = y \right)$ defines the language of strings with twice as many $\texttt{1}$'s as $\texttt{0}$'s. We allow linear equations or inequalities like $2x=y$, but only on count variables.
\end{enumerate}

\subsection{Definition}

We now describe the syntax of \logic, given a fixed finite alphabet~$\alphabet$, and its intended interpretation with reference to finite strings $w = w_1 \cdots w_n$ over $\alphabet$.
The syntax has two sorts \citep[p.~185--187]{immerman:1999}:
\begin{itemize}
\item The sort of \emph{positions} has variables $p, \dots$, which stand for positions of $w$, that is, integers in $[1, n]$.
\item The sort of \emph{counts} has
  \begin{itemize}
  \item variables $x, y, z, \dots$, which stand for rational numbers
  \item terms $c_0 + c_1 x_1 + \dots + c_k x_k$, where each $c_i$ is a rational number and each $x_i$ is a count variable.
  \end{itemize}
\end{itemize}

A formula of \logic{} is one of:
\begin{itemize}
\item $\true$ for true or $\false$ for false.
\item $\isletter{\letter}{p}$ where $\letter \in \alphabet$, which is true iff $w_p = \letter$.
\item $\congruent{r}{m}{p}$ where $r \ge 0, m > 0$, which is true iff $p \cong{m} r$.
\item $t_1 = t_2$, $t_1 < t_2$, where $t_1$ and $t_2$ are terms (in the sort of counts).
\item $\phi_1 \land \phi_2$, $\phi_1 \lor \phi_2$, $\neg \phi_1$ where $\phi_1$ and $\phi_2$ are formulas.
\item $\exists x. \phi$, $\forall x. \phi$ where $x$ is a count variable and $\phi$ is a formula.
\item $\counteq{x} p. \phi$, where $x$ is a count variable, $p$ is a position variable, and $\phi$ is a formula, which is true iff $\phi$ is true for exactly $x$ values of $p$. (Note that $\counteq{x} p$ binds $p$ but leaves $x$ free.)
\end{itemize}
Connectives $\phi \rightarrow \psi$ and $\phi \leftrightarrow \psi$ can be expressed as $\phi \rightarrow \psi \equiv \neg \phi \lor \psi$ and $\phi \leftrightarrow \psi \equiv \phi \rightarrow \psi \land \psi \rightarrow \phi$.
Quantifiers $\exists p. \phi$ and $\forall p. \phi$ can be expressed using $\exists x. (x > 0 \land \counteq{x} p. \phi)$ and $\exists x. (\counteq{x} p. \true \land \counteq{x} p. \phi)$, respectively.

A sentence is a formula with no free variables.
If $w \in \alphabet^*$ and $\sentence$ is a sentence, we write $w \models \sentence$ if $w$ makes $\sentence$ true under the intended interpretation.
\begin{definition}
  If $\sentence$ is a sentence of \logic{}, the language defined by $\sentence$ is $\{ w \mid w \models \sentence\}$.
\end{definition}

\subsection{Normal form}

The part of the logic having to do with position variables is like monadic first-order logic, in which all predicates are monadic (that is, they take only one argument).
The part of the logic having to do with count variables is the theory of rational numbers with ordering and addition (but not multiplication).
Both of these other logics have useful normal forms: monadic first-order logic has a normal form that uses only one variable \citep[p.~274--275]{boolos+:2007}, while the theory of rationals with ordering and addition has quantifier elimination \citep{robinson+zakon:1960,ferrante+rackoff:1975}.
We can combine these two results to get a very simple normal form for \logic{}.

\begin{theorem} \label{thm:normal_form}
  Every formula $\phi$ of \logic{} is equivalent to a formula of the form
  \begin{equation}
    \phi' = \exists x_1. \dots \exists x_k. \left( \bigwedge_i \counteq{x_i} p. \psi_i \land \chi \right)
  \end{equation}
  where
  \begin{itemize}
    \item Each $\psi_i$ is quantifier-free and has no free count variables.
    \item $\chi$ is quantifier-free.
  \end{itemize}
\end{theorem}

\begin{proof}
  See \cref{sec:normal_form}.
\end{proof}

\begin{toappendix}
\section{Proof of \cref{thm:normal_form}}
\label{sec:normal_form}

\Cref{thm:normal_form} says that 
every formula $\phi$ of \logic{} is equivalent to a formula of the form
  \begin{equation} \label{eq:normal_form}
    \phi' = \exists x_1. \dots \exists x_k. \left( \bigwedge_i \counteq{x_i} p. \psi_i \land \chi \right)
  \end{equation}
  where
  \begin{itemize}
    \item Each $\psi_i$ is quantifier-free and has no free count variables.
    \item $\chi$ is quantifier-free.
  \end{itemize}

The construction of $\phi'$ proceeds in three steps.
First, move counting quantifiers inward until they are of the form $\counteq{x_i}p.\psi_i$ where $\psi_i$ is quantifier-free and has no free count variables.
Second, move all subformulas $\counteq{x_i}p.\psi_i$ outward, leaving behind~$\chi$, as in \cref{eq:normal_form}.
Third, eliminate quantifiers from $\chi$.

The first step is analogous to the proof that a formula of monadic first-order logic can be converted to one with only one variable \citep[p.~274--275]{boolos+:2007}, which moves existential quantifiers inwards using the two following facts: $\exists x. (\phi \lor \psi) \equiv (\exists x. \phi) \lor (\exists x. \psi)$, and if $x$ does not occur free in $\phi$, then $\exists x. (\phi \land \psi) \equiv \phi \land (\exists x. \psi)$. With counting quantifiers, neither of the above holds, but the following two lemmas serve similar purposes.
\begin{lemma}
  \label{thm:distribute_counteq}
  Any formula $\phi = \counteq{x}p. \bigvee_{i=1}^l \phi_i$ is equivalent to a formula
  \begin{align}
   \phi' &= \exists x_1. \dots \exists x_m. \left( \bigwedge_{j=1}^m \counteq{x_j} p. \psi_j \land \chi \right) \label{eq:distribute_counteq}
  \end{align}
  where the $x_i$ are fresh count variables,
  each $\psi_j$ is a conjunction of a subset of the $\phi_i$, and $\chi$ has no free position variables.
\end{lemma}
\begin{proof}
  By induction on $l$. The case $l=1$ is trivial. For $l>1$, write
  \begin{align*}
    \phi &= \counteq{x}p. \Biggl(\,\bigvee_{i=1}^{l-1} \phi_i \lor \phi_l\Biggr) \\
    &\equiv \exists x_1. \exists x_2. \exists x_3. \Biggl(\counteq{x_1}p. \bigvee_{i=1}^{l-1} \phi_i \land \counteq{x_2}p.\phi_l \land \counteq{x_3} p. \left( \bigvee_{i=1}^{l-1} \phi_i \land \phi_l \right) \land x = x_1 + x_2 - x_3\Biggr) \\
    &\equiv \exists x_1. \exists x_2. \exists x_3. \Biggl({\underbrace{\counteq{x_1}p. \bigvee_{i=1}^{l-1} \phi_i}_{*}} \land \counteq{x_2}p.\phi_l \land {\underbrace{\counteq{x_3} p. \bigvee_{i=1}^{l-1} \left( \phi_i \land \phi_l \right)}_{*}} \land x = x_1 + x_2 - x_3\Biggr)
  \end{align*}
  and use the inductive hypothesis on the subformulas marked $*$ to put this into the form~(\ref{eq:distribute_counteq}).
\end{proof}

\begin{lemma} \label{thm:move_out_of_counteq}
  If $p$ does not occur free in $\phi$, then
  \[ \counteq{x} p. (\phi \land \psi) \equiv (\neg \phi \land x = 0) \lor (\phi \land \counteq{x}p. \psi). \]
\end{lemma}

Using these facts, we can prove the following:
\begin{proposition} \label{thm:clear}
  Every formula is equivalent to a formula in which
  \begin{itemize}
  \item in every subformula $\exists x. \phi$, there are no free position variables.
  \item in every subformula $\counteq{x} p. \phi$, $\phi$ is quantifier-free and the only free variable in $\phi$ is $p$ itself.
  \end{itemize}
\end{proposition}

\begin{proof} By induction on subformulas. Call a subformula \emph{clear} if it has the two properties listed above.
  
  \paragraph{Case $\exists x. \phi$:}
  Write $\phi$ as a Boolean combination of subformulas $\phi_1, \ldots, \phi_m$ that are either atomic or start with quantifiers. We can put $\phi$ into DNF in terms of the $\phi_i$ and distribute the $\exists x$ over the disjuncts. By the induction hypothesis, each of the $\phi_i$ is equivalent to a $\phi_i'$ which is clear. So $\phi \equiv \bigvee_{j=1}^m \exists x. \bigwedge_k \psi_{jk}$ where each $\psi_{jk}$ is one of the $\phi_i'$.
  Consider each $\psi_{jk}$.
  \begin{itemize}
  \item If $\psi_{jk}$ is $P(p)$ or $\neg P(p)$, then it can be moved out of the $\exists x$.
  \item If $\psi_{jk}$ starts with $\exists y$, $\neg \exists y$, $\counteq{y} p$, or $\neg \counteq{y} p$, then (because $\psi_{jk}$ is clear) it has no free position variables and does not need to be moved.
  \item If $\psi_{jk}$ is one of $y = z$, $y < z$, $y_1 + y_2 = z$, or their negations, then it has no free position variables and does not need to be moved.
  \end{itemize}
  Thus we have constructed a subformula equivalent to $\exists x. \phi$ that is clear.

  \paragraph{Case $\counteq{x} p. \phi$:}
  Again, write $\phi$ as a Boolean combination of subformulas $\phi_i$ that are either atomic or start with quantifiers, put $\phi$ into DNF in terms of the $\phi_k$. By the induction hypothesis, each of the $\phi_i$ is equivalent to a $\phi_i'$ which is clear. Use \cref{thm:distribute_counteq} to obtain \[ \phi \equiv \exists x_1. \cdots \exists x_m. \left( \bigwedge_j \counteq{x_j} p. \bigwedge_k \psi_{jk} \land \chi \right) \]
  where each $\psi_{jk}$ is one of the $\phi_i'$.
  Consider each $\psi_{jk}$.
  \begin{itemize}
  \item If $\psi_{jk}$ starts with $\exists y$, $\neg \exists y$, $\counteq{y} p$, or $\neg \counteq{y} p$, then it has no free position variables and can be moved out of the $\counteq{x_j}p$ using \cref{thm:move_out_of_counteq}.
  \item If $\psi_{jk}$ is one of $y = z$, $y < z$, $y_1 + y_2 = z$, or their negations, then it can be moved out of the $\counteq{x_j}p$ using \cref{thm:move_out_of_counteq}.
  \item If $\phi_{jk}$ is $P(q)$ or $\neg P(q)$ where $q \neq p$, then it can be moved out of the $\counteq{x}p$ using \cref{thm:move_out_of_counteq}, and also out of the $\exists x_i$.
  \item If $\phi_{jk}$ is $P(p)$ or $\neg P(p)$, then it only has free variable $p$ and does not need to be moved.
  \end{itemize}
  Thus we have constructed a subformula equivalent to $\counteq{x}p. \phi$ that is clear.
\end{proof}

This completes the first step.
For the second step, let $\chi$ be the formula obtained as follows: for every subformula $\counteq{x_i} p. \psi_i$, let $x_i'$ be a fresh count variable and replace the subformula with $x_i = x_i'$. Thus we have
\begin{equation*}
  \phi \equiv \exists x_1'. \dots \exists x_k'. \left( \bigwedge_i \counteq{x_i} p. \psi_i \land \chi \right)
\end{equation*}
which is almost in the desired form except that $\chi$ still has quantifiers.

The third step is the following:
\begin{theorem}[\citealp{ferrante+rackoff:1975}] \label{thm:qe}
  For any formula $\chi$ with no position variables (free or bound), there is a quantifier-free formula $\chi'$ equivalent to $\chi$.
\end{theorem}
Apply this procedure to $\chi$ and call the result $\chi'$.
Finally, let
\begin{equation*}
  \phi' = \exists x_1'. \dots \exists x_k'. \left(\bigwedge_i \counteq{x_i'} p. \psi_i \land \chi' \right). \tag*{\qedhere}
\end{equation*}

\end{toappendix}

It may seem odd that count variables range over rational numbers, when counts are always integers. This technicality simplifies the normal form: if we had used integers, then the part of the logic having to do with count variables would be Presburger arithmetic, and the normal form would require allowing $\congruent{r}{m}{x}$ on count variables as well.

\section{From Transformers to \logic}
\label{sec:upper}

In this section, we prove the following theorem, which sets an upper bound on the expressivity of fixed-precision transformer classifiers.

\begin{theorem} \label{thm:upper}
  Every language that is recognizable by a fixed-precision transformer classifier is definable by a sentence of \logic.
\end{theorem}

We don't specify exactly when and how a fixed-precision transformer performs rounding. Our translation to \logic{] for the most part can accommodate any rounding scheme, except that in \cref{eq:attention_as_average} below, we assume that the averages over positions $p$ are computed exactly, then rounded.

\subsection{Representing numbers}

Following \citet{merrill+sabharwal:2023}, we use a representation of real numbers with both limited precision and limited range. Limited precision might be justified by the fact that the numbers computed by a neural network are subject to noise (e.g., from randomness in sampling the training data and parameter optimization); limited range is justified by the observation by \citet{hahn:2020} that if all the functions used in a transformer are Lipschitz continuous, then the absolute value of all activations has an upper bound not depending on $n$. (\Cref{sec:layer_norm_upper} re-proves this result in the presence of layer normalization, which is not in general Lipschitz continuous.)

While \citet{merrill+sabharwal:2023} use floating point numbers ($p/2$ bits for the mantissa and $p/2$ bits for the exponent, where $p = 16$ or $32$), we use a fixed-point representation, with $\intbits$ for the integer part and $\fracbits$ for the fractional part. There is no loss of generality, because a floating-point number with a $p/2$-bit mantissa and $p/2$-bit exponent can be converted exactly to a fixed-point number with $p/2 + 2^{p/2}$ bits.

\begin{definition}
  A \emph{fixed-precision number} with $\intbits$ integer bits and $\fracbits$ fractional bits is a number in $\fixednums_{\intbits,\fracbits} = \{i/2^{\fracbits} \mid -2^{\intbits+\fracbits} \le i < 2^{\intbits+\fracbits} \}$.
  Since $\intbits$ and $\fracbits$ are fixed, we normally just write $\fixednums$ in place of $\fixednums_{\intbits,\fracbits}$.
\end{definition}

We write $\bit{i}{x}$ for the $i$-th bit in the two's-complement representation of $x$. That is,
\begin{equation*}
  \bit{i}{x} = \left\lfloor \frac{x}{2^i} \right\rfloor - 2 \left\lfloor \frac{x}{2^{i+1}} \right\rfloor
\end{equation*}
where $\lfloor x\rfloor$ is the greatest integer less than or equal to $x$.

A neural network, given input $\cls\cdot w$, computes many real-valued activations, which we can think of as functions $a \colon \alphabet^* \rightarrow \fixednums$. For each activation $a$, we will write sentences that test bits of $a(w)$.

\begin{definition}
  If $a \colon \alphabet^\ast \rightarrow \fixednums_{\intbits,\fracbits}$, we say that $a$ is \emph{defined} by sentences $\langle\sentence^a_k\rangle_{k \in [-\fracbits, \intbits]}$ (or just $\langle\sentence^a_k\rangle$ for short) if, for all $k \in [-\fracbits, \intbits]$, $w \models \sentence^a_k$ iff $\bit{k}{a(w)} = 1$.
  
    Similarly, if $\mathbf{a} \colon \alphabet^n \rightarrow \fixednums^{n'}$, we say that $\mathbf{a}$ is defined by $\langle \formula^{\mathbf{a}}_k[p], \omega^{\mathbf{a}}_k\rangle$ if
    $[\mathbf{a}(w)]_p$ is defined by $\langle \phi^{\mathbf{a}}_k[p]\rangle$ and
    $[\mathbf{a}(w)]_{0}$ is defined by $\langle \omega^{\mathbf{a}}_k\rangle$.
\end{definition}

The finiteness of $\fixednums$ ensures the following fact, which we use repeatedly:
\begin{proposition} \label{thm:fixed_function}
  If $a \colon \alphabet^\ast \rightarrow \fixednums$ is defined by $\langle\sentence^a_k\rangle$, then for any function $f \colon \fixednums \rightarrow \fixednums$ there are sentences that define $f \circ a$.
  Similarly, if $b \colon \alphabet^\ast \rightarrow \fixednums$ is defined by $\langle\sentence^b_k\rangle$, and $g \colon \fixednums \times \fixednums \rightarrow \fixednums$, there are sentences that define the function $g \circ (a, b)$ (that is, the function that maps $w$ to $g(a(w), b(w))$).
\end{proposition}
\begin{proof}
  Because $\fixednums$ is finite, it is easy but tedious to write sentences that test for all possible inputs and outputs.
\end{proof}

\subsection{Input layer} \label{sec:input_layer}

The (function that maps $w$ to the) $i$-th component of $\text{WE}(w_p)$ is defined by
\begin{align*}
  \formula_k[p] &= \bigvee_{\mathclap{\substack{\letter \in \alphabet \\ \bit{k}{[\text{WE}(\letter)]_i} = 1}}} \isletter{\letter}{p}
\end{align*}
and $\omega_k$, which simply encodes the constant value $[\text{WE}(\cls)]_i$ in fixed-precision.

Sinusoidal PEs, rounded to the nearest fixed-precision number, can be described using modular predicates.
For any~$i$, there exists a period $m_i$ such that for all $p$, $[\text{PE}(p)]_i = [\text{PE}(p+m_i)]_i$.
Then the (function that maps $w$ to the) $i$-th component of $\text{PE}(p)$ is defined by
\begin{align*}
\formula_k[p] &= \bigvee_{\mathclap{\substack{0 \leq r < m_i \\ \bit{k}{[\text{PE}(r)]_i} = 1}}} \congruent{r}{m_i}{p}
\end{align*}
and $\omega_k$, which simply encodes the constant value $[\text{PE}(0)]_i$ in fixed-precision.

\subsection{Hidden layers} \label{sec:hidden_layers}

The position-wise FFNs and residual connections can all be defined using \cref{thm:fixed_function}.

In a self-attention layer, the logits \labelcref{eq:logits} can be defined using \cref{thm:fixed_function}. The equation for context vectors \labelcref{eq:context} can be rewritten in terms of averages (not sums, which could overflow):
\begin{align} \label{eq:attention_as_average}
  c_q &= \frac{\frac1{n'}\sum_p \exp s_{qp} \wv A_{*,p}}{\frac1{n'}\sum_{p} \exp s_{qp}}.
\end{align}
Everything in this equation except for averaging ($\frac1{n'}\sum_{p}$) can also be defined using \cref{thm:fixed_function}.
So the one operation that remains to be defined is averaging over $n'$ positions.
\Cref{sec:extra_prec} explains how to do this, one bit at a time. To sum bits across all positions, we use counting quantifiers; to divide by $n'$, we use another kind of number representation whose least-significant digits are $n'$ times smaller than those of fixed-precision numbers.

\begin{toappendix}
\section{Expressing averages in \logic{}}
\label{sec:extra_prec}

To compute averages across $n$ positions, we define a new kind of numeric representation, called \emph{extra-precision numbers}, that have an \emph{extra digit}, ranging from $0$ to $n$ with $\frac{1}{n'}$ of the place value of the least-significant bit of a fixed-precision number.

\begin{definition}
  An \emph{extra-precision number} with $\intbits$ integer bits, $\fracbits$ fractional bits, and extra digit with base $n'$ is a number in $\extranums_{\intbits,\fracbits,n'} = \left\{ \frac{i}{2^\fracbits n'} \bigm| -2^{\intbits+\fracbits}n' \le i < 2^{\intbits+\fracbits}n' \right\}$. For any $a \in \extranums_{\intbits, \fracbits, n'}$, we define
  \begin{align*}
    \extra{a} &= \floor{a \cdot 2^\fracbits \cdot n'} - \floor{a \cdot 2^\fracbits} \cdot n' \\
    \fixed{a} &= \floor{a \cdot 2^\fracbits} \cdot 2^{-\fracbits}.
  \end{align*}
\end{definition}

For example, in $\extranums_{1,2,3}$, the extra digit has place value $\frac{1}{2^2 \cdot 3} = \frac{1}{12}$. The number $\frac{17}{12}$ belongs to $\extranums_{1,2,3}$ and can be represented as $01.012$, because $\frac{17}{12} = 1 \cdot 1 + 0 \cdot \frac12 + 1 \cdot \frac14 + 2 \cdot \frac1{12}$. Then $\extra{\frac{17}{12}} = 2$ and $\fixed{\frac{17}{12}} = \frac{5}{4}$ (or $01.01$ in binary). Its negation, $-\frac{17}{12}$, can be represented as $10.101$. The sign bit can be thought of as having place value $-2$, and $-\frac{17}{12} = -1 \cdot 2 + 0 \cdot 1 + 1 \cdot \frac12 + 0 \cdot \frac14 + 1 \cdot \frac1{12}$.

\begin{definition}
  If $a \colon \alphabet^* \rightarrow \extranums_{\intbits,\fracbits,n'}$, we say that $a$ is \emph{defined} by sentences $\sentence^a_k$ for $k \in [-\fracbits, \intbits]$ and $\formula^a_\extradigit[x]$ (or $\langle \sentence^a_k, \formula^a_\extradigit[x]\rangle$ for short) if $\langle\sentence^a_k\rangle$ defines $\fixed{a(w)}$, and $w \models \formula^a_\extradigit[x]$ iff $x = \extra{a(w)}$.
\end{definition}

Adding or subtracting two extra-precision numbers is possible because \logic{} has addition of count variables:
\begin{proposition} \label{thm:extra_addition}
  Let $a \colon \alphabet^* \rightarrow \extranums_{\intbits,\fracbits,n'}$ be defined by $\langle\sentence^a_k, \formula^a_\extradigit[x]\rangle$ and $b \colon \alphabet^* \rightarrow \extranums_{\intbits,\fracbits,n'}$ be defined by $\langle\sentence^b_k, \formula^b_\extradigit[x]\rangle$. If, for all $w$, $(a+b)(w) \in \extranums_{\intbits,\fracbits,n'}$, then $(a+b)$ is defined by some $\langle\sentence^{a+b}_k, \formula^{a+b}_\extradigit\rangle$, and similarly for $(a-b)$.
\end{proposition}
\begin{proof}
First, the extra digit of $(a+b)(w)$ is obtained by adding the extra digits of $a(w)$ and $b(w)$; if the sum is $n'$ or more, then subtract $n'$ and carry a $1$ to the least-significant bit. Thus, let
\begin{align*}
  \formula^{a+b}_\extradigit[z] &= \exists x. \exists y. \exists n. \Bigl(\begin{aligned}[t] &\formula^a_\extradigit[x] \land \formula^b_\extradigit[y] \land \counteq{n} p. \true \\
  &\quad \land \bigl(\underbrace{(x+y<n+1 \land z=x+y)}_{\text{no carry}} {} \lor \underbrace{(x+y\ge n+1 \land z=x+y-(n+1))}_{\text{carry}}\bigr) \Bigr).\end{aligned}
\end{align*}
As for the remaining bits, by \cref{thm:fixed_function}, $\fixed{a}+\fixed{b}$ is definable by some $\langle\sentence^0_k\rangle$. Similarly, if there was a carry from the extra digit, $\fixed{a}+\fixed{b}+1$ is definable by some $\langle\sentence^1_k\rangle$. Then let
\begin{align*}
  \sentence^{a+b}_k &= \exists x. \exists y. \exists n. \Bigl( \formula^a_\extradigit[x] \land \formula^b_\extradigit[y] \land \counteq{n} p. \true
  \land \bigl( \underbrace{(x + y < n+1 \land \sentence^0_k)}_{\text{no carry}} {} \lor \underbrace{(x + y \ge n+1 \land \sentence^1_k)}_{\text{carry}} \bigr) \Bigr).
\end{align*}

To get subtraction, it suffices to define negation. Recall that in two's complement representation, negation means inverting all the digits and adding 1. For the extra digit, ``inverting'' means subtracting from $n$. If the extra digit is $0$, invert it to get $n$, then increment it to $0$ with a carry to the least-significant bit.
\begin{align*}
  \formula^{-b}_\extradigit[z] &= \exists y. \exists n. \Bigl(\formula^b_\extradigit[y] \land \counteq{n} p. \true \land \bigl({\underbrace{(y > 0 \land z = n+1 - y)}_{\text{no carry}}} \lor {\underbrace{(y = 0 \land z = 0)}_{\text{carry}}}\bigr) \Bigr).
\end{align*}
Write $\bitneg$ for bitwise negation.
By \cref{thm:fixed_function}, $\bitneg\fixed{b}$ is definable by some $\langle\sentence^0_k\rangle$, and $\bitneg\fixed{b}+1$ is definable by some $\langle\sentence^1_k\rangle$. Then
\begin{align*}
  \sentence^{-b}_k &= \exists y. \exists n. \left(\formula^b_\extradigit[y] \land \counteq{n} p. \true \land \bigl({\underbrace{(y > 0 \land \sentence^0_k)}_{\text{no carry}}} \lor {\underbrace{(y = 0 \land \sentence^1_k)}_{\text{carry}}} \bigr) \right). \tag*{\qedhere}
\end{align*}
\end{proof}

Finally, we can show how to define averages over $n'$ positions.
\begin{proposition}  
  Given a function $\mathbf{a} : \alphabet^n \rightarrow \fixednums_{\intbits,\fracbits}^{n'}$,
  defined by $\langle\formula^\mathbf{a}_k[p], \omega^{\mathbf{a}}_k\rangle$,
  the function
  \begin{align*}
    \bar{\mathbf{a}} \colon \alphabet^* &\rightarrow \extranums_{\intbits,\fracbits,n'} \\
    w &\mapsto \frac1{n'} \sum_{i=1}^{n'} [\mathbf{a}(w)]_p
  \end{align*}
  is approximated with error at most $2^{-\fracbits}$ by a function that
  is definable by some $\langle \sentence^{\bar{\mathbf{a}}}_k\rangle$.
\end{proposition}
\begin{proof}
  Observe that
  \begin{align}
    [\mathbf{a}(w)]_p &= -\bit{\intbits}{[\mathbf{a}(w)]_p} \cdot 2^\intbits + \sum_{k=-\fracbits}^{\intbits-1} \bit{k}{[\mathbf{a}(w)]_p} \cdot 2^k \notag \\
\intertext{so the average can be written as}
    \bar{\mathbf{a}}(w) &= \frac1{n'} \sum_{p=0}^{n'} [\mathbf{a}(w)]_p \notag \\
    &= -\frac1{n'} \sum_{p=0}^{n'} \bit{\intbits}{[\mathbf{a}(w)]_p} \cdot 2^\intbits + \sum_{k=-\fracbits}^{\intbits-1} \frac1{n'} \sum_{p=0}^{n'} \bit{k}{[\mathbf{a}(w)]_p} \cdot 2^k \notag \\
    &= -\underbrace{\frac1{2^{\fracbits}n'} \sum_{p=0}^{n'} \bit{\intbits}{[\mathbf{a}(w)]_p}}_{v_\intbits} {} \cdot 2^{\intbits+\fracbits} + \sum_{k=-\fracbits}^{\intbits-1} {} \underbrace{\frac1{2^{\fracbits}n'} \sum_{p=0}^{n'} \bit{k}{[\mathbf{a}(w)]_p}}_{v_k} {} \cdot 2^{k+\fracbits}. \label{eq:average}
  \end{align}
  Each $v_k$ is the sum of all the $k$-th bits, written in the extra digit's place. It is defined by
  \begin{align*}
    \sentence^{v_k}_k &= \false \\
    \formula^{v_k}_\extradigit[x] &= \exists y. \left ( (\omega^{\mathbf{a}}_k  \land x = y+1 \lor \neg \omega^{\mathbf{a}}_k \land x = y) \land \counteq{y} p.\formula^{\mathbf{a}}_k[p] \right).
  \end{align*}
  Then use \cref{thm:extra_addition} to define \cref{eq:average}. The multiplications by powers of 2 can be accomplished by repeated addition, as can the summation over $k$.
\end{proof}
\end{toappendix}

\subsection{Output layer} \label{sec:output_layer}

By composing the constructions from the preceding sections, we obtain formulas $\langle \phi_k[p], \omega_k\rangle$ that define the output of the encoder at all positions. We are only interested in the output at \cls, which is defined by $\langle \omega_k\rangle$, so we can discard the $\phi_k[p]$.
By \cref{thm:fixed_function}, we can define the top-level sentence, which applies the final sigmoid layer \labelcref{eq:output_layer} to $\langle \omega_k\rangle$ and tests whether the result is at least $\tfrac12$.

\subsection{Complexity analysis}

The sentence constructed by \cref{thm:upper} would be quite large if written out in full, because of repeated subformulas. We analyze its size assuming that repeated subformulas can share space.

In the input layer (\cref{sec:input_layer}), the word embeddings translate to subformulas with total size $O(|\alphabet|d(\intbits+\fracbits))$, and the positional encodings, $O(md(\intbits+\fracbits))$, where $m$ is the maximum period of any component of $\text{PE}$.
In the hidden layers (\cref{sec:hidden_layers}), the position-wise FFNs and residual connections translate into subformulas with total size $O(L d^2 F)$, where $F$ is the maximum size of a subformula constructed by \cref{thm:fixed_function}, which in the worst case could be exponential in the precision ($\intbits+\fracbits$).
The attention layers translate into subformulas with total size $O(L H d^2 F)$. Finally, the output layer (\cref{sec:output_layer}) translates into a subformula of size~$O(d^2 F)$.

\subsection{Relationship to uniform \tczero}
\label{sec:tc0}

We conclude this section by showing that \logic{} is strictly less expressive than uniform \tczero and therefore a tighter upper bound on fixed-precision transformer encoders than that of \citet{merrill+sabharwal:2023}. (On the other hand, their proof applies to a much more general class of neural networks than ours does.)

(Non-uniform) \tczero{} is the class of families of Boolean circuits with majority gates, unlimited fan-in, polynomial size, and constant depth. By \emph{uniform} \tczero{} we mean circuit families in~\tczero{} whose connections can be decided in logarithmic time \citep{barrington+:1990}.

\begin{proposition} \label{thm:tc0}
  The language {\normalfont $\{\texttt{0}^n \texttt{1}^n \mid n \ge 0\}$} is in uniform \tczero{} but not definable in \logic.
\end{proposition}

\begin{proof}
  For inclusion in uniform $\tczero$, we use the fact that uniform $\tczero$ is equivalent to first-order logic with majority quantifiers, addition, and multiplication, and that majority quantifiers can simulate counting quantifiers \citep[p.~296]{barrington+:1990}. Then $\{\texttt{0}^n \texttt{1}^n\}$ is defined by
  \begin{align*}
    &\exists x. \bigl(\counteq{x} p. \isletter{\texttt{0}}{p} \land \counteq{x} p. \isletter{\texttt{1}}{p}\bigr) \\ & \quad \land \forall p. \forall q. \bigl(\isletter{\texttt{0}}{p} \land \isletter{\texttt{1}}{q} \rightarrow p < q\bigr).
  \end{align*}

  For non-definability in \logic, suppose that $\{\texttt{0}^n \texttt{1}^n\}$ is definable in \logic{} by some sentence~$\sentence$. Let $M$ be the product of all moduli $m$ used in atomic formulas $\congruent{m}{r}{p}$ used in $\sentence$. Then $\sentence$ cannot distinguish between positions $p$ and $(p+M)$, so it cannot distinguish $w = \texttt{0}^M \texttt{1}^M$ and $w' = \texttt{1} \texttt{0}^{M-1} \texttt{0} \texttt{1}^{M-1}$. Since $w \models \sentence$, it must be the case that $w' \models \sentence$, which is a contradiction.
\end{proof}

\section{From \logic{} to Transformers}
\label{sec:lower}

In this section, we prove the following theorem, which sets a lower bound on the expressivity of (arbitrary-precision) transformer classifiers.

\begin{theorem} \label{thm:lower}
  Every language that is definable by a sentence of \logic{} is also recognizable by a transformer classifier.
\end{theorem}

By \cref{thm:normal_form}, we can assume \[ \sentence \equiv \exists x_1. \dots \exists x_k. \left( \bigwedge_i \counteq{x_i}p. \psi_i[p] \land \chi[x_1, \dots, x_k] \right) \] where every $\psi_i$ is quantifier-free with one free position variable and no free count variables, and $\chi$ is quantifier-free with no free position variables.

Then the proof constructs a transformer classifier with three parts.
The first, lowest, part of the network computes the truth values of the $\psi_i[p]$ at every position $p$. The second part uses uniform self-attention to find each $x_i$, the number of positions $p$ that make $\psi_i[p]$ true. The third part computes the truth value of $\chi$ and of the whole sentence.

We first show how to do this without layer normalization; \cref{sec:layer_norm_lower} explains how to modify the construction for layer normalization.

\subsection{Computing the $\psi_i[p]$}
\label{sec:psi}

For each $\psi_i[p]$, we construct a transformer encoder that computes its truth-value, in the following sense:
{\renewcommand{\formula}{\psi}
\begin{lemma} \label{thm:lower_lemma1}
  For any formula $\formula[p]$ of \logic{} which is quantifier-free with exactly one free position variable $p$ and no free count variables, there is a transformer encoder $T$ with width $d$ such that, for all $w \in \alphabet^*$ and $p \in [1, |w|]$, $[T(w)]_{d,p} = \indicator{w \models \formula[p]}$ and $[T(w)]_{d,0} = 0$.
\end{lemma}

The proof, given in \cref{sec:lower_lemma1}, is by induction on subformulas.
The cases for $\formula_1 \land \formula_2$ and $\formula_1 \lor \formula_2$ invoke the inductive hypothesis for both $\formula_1$ and $\formula_2$ and combine them into a single transformer encoder using the following:
\begin{lemma} \label{thm:concat}
  If\/ $T_1$ and $T_2$ are transformer encoders,
  then there is a transformer encoder, called $T_1 \oplus T_2$, such that
  \[ (T_1 \oplus T_2)(w) = \begin{bmatrix} T_1(w) \\ T_2(w) \end{bmatrix}. \]
\end{lemma}

\begin{proof}
  See \cref{sec:concat}.
\end{proof}

Then most cases add one hidden layer on top, whose self-attention does nothing ($\wv = \mathbf{0}$) and whose FFN computes the relevant function.

Thus, for each $\psi_i[p]$, we get an equivalent transformer encoder, which we call~$\Psi_i$.
} 

\begin{toappendix}
\subsection{Proof of \cref{thm:concat}}
\label{sec:concat}

  If one encoder is less deep than the other, add layers to it that compute the identity function:
  For the self-attention, just set $\wv = \mathbf{0}$, and for the FFN, set $\wone = \mathbf{0}$ and $\bone = \mathbf{0}$. In both cases, the sublayer computes the identity function thanks to the residual connections.
  
  Concatenate the word and position vectors:
  \begin{align*}
    \text{WE}(a) &= \begin{bmatrix} \text{WE}_1(a) \\ \text{WE}_2(a) \end{bmatrix} &
    \text{PE}(p) &= \begin{bmatrix} \text{PE}_1(p) \\ \text{PE}_2(p) \end{bmatrix}.
  \end{align*}
  Although we only ever concatenate layers whose self-attentions compute the identity function, we show how to concatenate self-attentions for completeness.
  For each pair of multi-head self-attentions, $\selfatt_1^{(1)}, \ldots, \selfatt_1^{(H_1)}$ and $\selfatt_2^{(1)}, \ldots, \selfatt_2^{(H_2)}$, create a multi-head self-attention with $H_1+H_2$ heads:
  \begin{align*}
    W^{(h,\text{Q})} &= \begin{bmatrix} W_1^{(h,\text{Q})} & \mathbf{0} \end{bmatrix} \quad (1 \leq h \leq H_1) & W^{(h,\text{Q})} &= \begin{bmatrix} \mathbf{0} & W_2^{(h-H_1,\text{Q})} \end{bmatrix} \quad (H_1+1 \leq h \leq H_1+H_2) \\
    W^{(h,\text{K})} &= \begin{bmatrix} W_1^{(h,\text{K})} & \mathbf{0} \end{bmatrix} & W^{(h,\text{K})} &= \begin{bmatrix} \mathbf{0} & W_2^{(h-H_1,\text{K})} \end{bmatrix} \\
    W^{(h,\text{V})} &= \begin{bmatrix} W_1^{(h,\text{V})} & \mathbf{0} \\ \mathbf{0} & \mathbf{0} \end{bmatrix} & W^{(h,\text{V})} &= \begin{bmatrix} \mathbf{0} & \mathbf{0} \\ \mathbf{0} & W_2^{(h-H_1,\text{V})} \end{bmatrix}
  \end{align*}

  For each $\text{FFN}_1$ and $\text{FFN}_2$, create a FFN:
  \begin{align*}
    \wone &= \begin{bmatrix} W_1^{(1)} & \mathbf{0} \\ \mathbf{0} & W_2^{(1)} \end{bmatrix} &
    \bone &= \begin{bmatrix} b_1^{(1)} \\ b_2^{(1)} \end{bmatrix} \\
    \wtwo &= \begin{bmatrix} W_1^{(2)} & \mathbf{0} \\ \mathbf{0} & W_2^{(2)} \end{bmatrix} &
    \btwo &= \begin{bmatrix} b_1^{(2)} \\ b_2^{(2)} \end{bmatrix}.
  \end{align*}
\end{toappendix}

\begin{toappendix}
{\renewcommand{\formula}{\psi}
\subsection{Proof of \cref{thm:lower_lemma1}}
\label{sec:lower_lemma1}

The following lemma implies that we can always modify a FFN to cancel out its residual connection:
\begin{lemma} \label{thm:cancel_res}
  If $f \colon \R^d \rightarrow \R^d$ is a FFN, there is a FFN $f_-$ such that $f_-(x) = f(x) - x$.
\end{lemma}
\begin{proof}
  \begin{align*}
    W_-^{(1)} &= \begin{bmatrix} \wone \\ \mathbf{I} \\ -\mathbf{I} \end{bmatrix} & b_-^{(1)} &= \begin{bmatrix} \bone \\ \mathbf{0} \\ \mathbf{0} \end{bmatrix} \\
    W_-^{(2)} &= \begin{bmatrix} \wtwo & -\mathbf{I} & \mathbf{I} \end{bmatrix} & b_-^{(2)} &= \btwo. \tag*{\qedhere}
  \end{align*}
\end{proof}

To prove \cref{thm:lower_lemma1}, we first prove the following slightly modified claim, by induction on subformulas:
\begin{lemma}
  For any formula $\formula[p]$ of \logic{} which is quantifier-free with exactly one free position variable $p$ and no free count variables, there is a transformer encoder $T$ with width $d$ such that, for all $w \in \alphabet^*$ and $p \in [1, |w|]$, $[T(w)]_{d,p} = \indicator{w \models \formula[p]}$.
\end{lemma}
Then the final step will be to set $[T(w)]_{d,0} = 0$.

In the following cases, we give diagrams of FFNs rather than writing out their weight matrices. In these diagrams, a node~\begin{tikzpicture}[baseline=-3pt] \node[input]; \end{tikzpicture}
with a value underneath it stands for the component of the input vector that contains that value. A node~\begin{tikzpicture}[baseline=-3pt] \node[relu]; \end{tikzpicture} is a ReLU unit, and a node~\begin{tikzpicture}[baseline=-3pt] \node[linear]; \end{tikzpicture} is a linear unit. Edges are connections, with their connection weights written next to them. If a unit has nonzero bias, the bias is written next to it. When the residual connection is not cancelled out using~\cref{thm:cancel_res}, we draw the residual connection as an edge with weight~1.

  \paragraph{Case $\formula = \isletter{\letter}{p}$:} Construct a transformer encoder with just an input layer ($L=0$):
  \begin{align*}
    \text{WE}(\letter') &= \begin{bmatrix} \indicator{\letter'=\letter} \\ 0 \end{bmatrix} \\
    \text{PE}(p) &= \begin{bmatrix} \sin 0 \\ \cos 0 \end{bmatrix} = \begin{bmatrix} 0 \\ 1 \end{bmatrix}.
  \end{align*}
  Then add one hidden layer, whose self-attention does nothing ($\wv = \mathbf{0}$) and whose FFN is
  \begin{center}
    \begin{tikzpicture}
      \node (i) at (0,0) [input,label=below:{$\letter'=\letter$}];
      \node at (1,0) [input,label=below:$1$];
      \node (h) at (1,1) [relu] edge node {$1$} (i);
      \node (o) at (1,2) [linear] edge node {$1$} (h);
    \end{tikzpicture}
  \end{center}
  Apply \cref{thm:cancel_res} to cancel out the residual connection. 
  \paragraph{Case $\formula = \congruent{r}{m}{p}$:}
  Any function on integer positions with period $m$ can be expressed as a linear function of a sinusoidal PE with width $2m$ using a discrete Fourier series. But here we exploit ReLUs to give a more compact network.
  Let
  \begin{align*}
    \text{WE}(\letter) &= \begin{bmatrix} 0 \\ 0 \end{bmatrix} \\
    \text{PE}(p) &= \begin{bmatrix} \sin \frac{2\pi p}{m} \\ \cos \frac{2\pi p}{m}  \end{bmatrix}.
  \end{align*}
  Then add one hidden layer, whose self-attention does nothing, and whose FFN is
  \begin{center}
    \begin{tikzpicture}[x=1.5cm]
      \node (i1) at (0,0) [input,label=below:$\sin\frac{2\pi p}{m}$];
      \node (i2) at (1,0) [input,label=below:$\cos\frac{2\pi p}{m}$];
      \node (h) at (1,1) [relu,bias={$-\cos\frac{2\pi}{m}$}] edge node {$\sin \frac{2\pi r}{m}$} (i1) edge node[auto=left] {$\cos \frac{2\pi r}{m}$} (i2);
      \node (o) at (1,2) [linear] edge node {$\frac{1}{1-\cos\frac{2\pi}{m}}$} (h);
    \end{tikzpicture}
  \end{center}
  Apply \cref{thm:cancel_res} to cancel out the residual connection.
  So (using the identity $\cos (x-y) = \cos x \cos y + \sin x \sin y$) the output vectors are
  \begin{align*}
    [T(w)]_{*,p} = \begin{bmatrix}
      0 \\
      \frac{ \max \left\{0, \cos \frac{2\pi(p-r)}{m} - \cos \frac{2\pi}{m} \right\}}{1-\cos\frac{2\pi}{m}} 
    \end{bmatrix} \\
    \intertext{and for $p$ an integer, this simplifies to}
    [T(w)]_{*,p} = \begin{bmatrix} 0 \\ \indicator{p \cong{m} r} \end{bmatrix}.
  \end{align*}

  For example, the graph below shows the case $r=1$, $m=5$:
  \begin{center}
    \def\r{1} \def\m{5}
    \tikzset{every axis/.append style={xmin=1,xmax=10,ymin=-1.5,ymax=1.5,width=1.75in,height=1.25in}}
    \begin{tikzpicture}[baseline=0]
      \begin{axis}[xlabel={$p$},ylabel={$[T(w)]_{2,p}$}]
        \plot[mark=none,domain={0:10},samples=100] { (max(0,cos(deg(2*pi*(x-\r)/\m))-cos(deg(2*pi/\m)))) / (1-cos(deg(2*pi/\m))) };
        \plot[mark=*,draw=none,domain={0:10},samples=11] { (max(0,cos(deg(2*pi*(x-\r)/\m))-cos(deg(2*pi/\m)))) / (1-cos(deg(2*pi/\m))) };
      \end{axis}
    \end{tikzpicture}
  \end{center}

  \paragraph{Case $\formula = \neg \formula_1$:} By the induction hypothesis, let $T_1$ be a transformer encoder that computes $\formula_1$.
  Add one new layer. The self-attention does nothing ($\wv = \mathbf{0}$), and the FFN performs the negation:
  \begin{center}
    \begin{tikzpicture}
      \node (x) at (0,0) [input,label=below:$\formula_1$];
      \node (h) at (0,1) [relu] edge node {$1$} (x);
      \node (o) at (0,2) [linear,bias={$1$}] edge node {$-1$} (h);
    \end{tikzpicture}
  \end{center}
  Then apply \cref{thm:cancel_res} to cancel out the residual connection. This computes the negation in the last dimension.
  
  \paragraph{Case $\formula = \formula_1 \land \formula_2$:} By the induction hypothesis, let $T_1$ and $T_2$ be transformer encoders that compute $\formula_1$ and $\formula_2$ (respectively).
  By Lemma~\ref{thm:concat}, we construct $T_1 \oplus T_2$.
  Then we add one new layer.
  The self-attention does nothing ($\wv = \mathbf{0}$), and the FFN computes the minimum of its two inputs. In this case, we do not use \cref{thm:cancel_res}, as we make use of the residual connection.
  \begin{center}
    \begin{tikzpicture}
      \node(x) at (0,0) [input,label=below:$\formula_1$];
      \node(y) at (1,0) [input,label=below:$\formula_2$];
      \node(h) at (0,1) [relu] edge node[auto=right]{$-1$} (x) edge node[auto=left,near start]{$1$} (y);
      \node(out) at (1,2) [linear] edge node[auto=right]{$-1$} (h) edge node[auto=left] {$1$} (y);
    \end{tikzpicture}
  \end{center}

  \paragraph{Case $\formula = \formula_1 \lor \formula_2$:} The FFN computes the maximum of its two inputs:
  \begin{center}
    \begin{tikzpicture}
      \node(x) at (0,0) [input,label=below:$\formula_1$];
      \node(y) at (1,0) [input,label=below:$\formula_2$];
      \node(h) at (0,1) [relu] edge node[auto=right]{$1$} (x) edge node[auto=left,near start]{$-1$} (y);
      \node(out) at (1,2) [linear] edge node[auto=right]{$1$} (h) edge node[auto=left] {$1$} (y);
    \end{tikzpicture}
  \end{center}

  Finally, to set $[T(w)]_{d,0} = 0$, construct a transformer encoder $N$ such that, for all $p$, $[N(w)]_{*,p} = \begin{bmatrix} \indicator{w_p \neq \cls} \end{bmatrix}$,
and use \cref{thm:concat} to form $T \oplus N$. Add a new layer whose self-attention does nothing, and whose FFN computes the minimum, just as in the case for $\land$ above.
} 
\end{toappendix}

\subsection{Counting quantifiers}

The second step is to find the value of each $x_i$, which is the number of positions $p$ for which $w \models \psi_i[p]$.
Construct a transformer encoder $U$ such that, for all $p$,
\begin{equation*}
  [U(w)]_{*,p} = \begin{bmatrix}
    \indicator{w_p = \cls} \\
    \mathbf{0}^{k+1}
  \end{bmatrix}.
\end{equation*}
Use \cref{thm:concat} to form $\Psi = {\bigoplus_{i=1}^k \Psi_i} \oplus U$, and add one more layer.
In the self-attention, $\wv$ projects and permutes dimensions so that the value vectors are
\begin{equation*}
  \wv\,[\Psi(w)]_{*,p} = 
  \begin{bmatrix}
    \mathbf{0} \\
    \indicator{w\models\psi_1[p]} \\
    \vdots \\
    \indicator{w\models\psi_k[p]} \\
    \indicator{w_p = \cls}
  \end{bmatrix}.
\end{equation*}
The self-attention uses uniform attention to average over~$p$ ($\wq = \wk = \mathbf{0}$), and the FFN does nothing ($\wone = \bone = \wtwo = \btwo = \mathbf{0}$). Call this encoder $C$. Its output, for all $p$, is
\begin{equation*}
  [C(w)]_{*,p} = \frac{1}{n'} \begin{bmatrix}
    \mathbf{0} \\
    x_1 \\
    \vdots \\
    x_k \\
    1
  \end{bmatrix}.
\end{equation*}

\subsection{Computing $\chi$}
\label{sec:chi}

The third step is to compute the truth value of $\chi$ given the values of the $x_i$.
After the division by $n'$, we can no longer use 1 for true and 0 for false; instead, we use positive numbers for true and negative numbers for false.

\begin{lemma} \label{thm:lower_lemma2}
  If $\chi$ is a quantifier-free formula of \logic{} with free count variables $x_1, \dots, x_k$ and no position variables,
  then there is a transformer stack $\Chi$ with width $d$ such that for any $A \in \R^{d \times n'}$ with
  \[ A_{*,0} = \alpha \begin{bmatrix} \mathbf{0} \\ x_1 \\ \vdots \\ x_k \\ 1 \end{bmatrix} \qquad x_1, \ldots, x_k \in \mathbb{Z}\]
  we have $[\Chi(A)]_{d,0} > 0$ if $\chi[x_1, \dots, x_k]$ is true, and $[\Chi(A)]_{d,0} < 0$ if $\chi[x_1, \dots, x_k]$ is false.
\end{lemma}

\begin{proof}
  See \cref{sec:lower_lemma2}.
\end{proof}

\begin{toappendix}
\subsection{Proof of \cref{thm:lower_lemma2}}
\label{sec:lower_lemma2}

By induction on the structure of $\chi$. For simplicity, we set $\alpha=1$; since all the FFNs do not have bias, the construction will work for any $\alpha>0$ as well.

\paragraph{Case $c_0 + \sum_i c_i x_i > 0$:}
  Add a new layer whose self-attention does nothing and whose FFN (after applying \cref{thm:cancel_res}) computes the piecewise linear function shown at right.
  \begin{center}
    \begin{tabular}{cc}
      \begin{tikzpicture}[x=2cm,y=1.5cm]
        \node(x1) at (0,0) [input,label=below:$x_1$];
        \node at (0.5,0) {$\cdots$};
        \node(xk) at (1,0) [input,label=below:$x_k$];
        \node(one) at (2,0) [input,label=below:$1$];
        \node(h1) at (0,1) [relu] edge node[near start] {$c_1$} (x1) edge node[near start] {$c_k$} (xk) edge node[auto=left,very near start] {$c_0$} (one);
        \node(h2) at (1,1) [relu] edge node[very near start] {$c_1$} (x1) edge node[very near start] {$c_k$} (xk) edge node[near start,auto=left] {$-1$} (one);
        \node(h3) at (2,1) [relu] edge node[near start,auto=left] {$1$} (one);
        \node(out) at (2,2) [linear] edge node {$2$} (h1) edge node[auto=left] {$-2$} (h2) edge node[auto=left] {$-1$} (h3);
      \end{tikzpicture}
      &
      \begin{tikzpicture}
        \begin{axis}[xmin=-2,xmax=2,ymin=-1.2,ymax=1.2,height=5cm,axis equal image,xlabel={$c_0 + \sum_i c_i x_i$},xtick={-2,-1,0,1,2},xticklabels={$-2$,$-1$,$0$,$1$,$2$},ytick={-1,0,1},yticklabels={$-1$,$0$,$1$}]
          \addplot[mark=none] coordinates { (-2,-1) (0,-1) (1,1) (2,1) };
        \end{axis}
      \end{tikzpicture}
    \end{tabular}
  \end{center}
  
  \paragraph{Case $c_0 + \sum_i c_i x_i = 0$:} Same as above, but with the following FFN.
  \begin{center}
    \begin{tabular}{cc}
      \begin{tikzpicture}[x=3cm,y=1.5cm]
        \node(x1) at (0,0) [input,label=below:$x_1$];
        \node at (0.5,0) {$\cdots$};
        \node(xk) at (1,0) [input,label=below:$x_k$];
        \node(one) at (2,0) [input,label=below:$1$];
        \node(h1) at (0,1) [relu] edge node[near start] {$c_1$} (x1) edge node[very near start] {$c_k$} (xk) edge node[auto=left,pos=0.05] {$c_0+1$} (one);
        \node(h2) at (0.66,1) [relu] edge node[pos=0.1] {$c_1$} (x1) edge node[auto=left,pos=0.15] {$c_k$} (xk) (one);
        \node(h3) at (1.33,1) [relu] edge node[pos=0.1] {$c_1$} (x1) edge node[auto=left,very near start] {$c_k$} (xk) edge node[very near start,auto=left] {$-1$} (one);
        \node(h4) at (2,1) [relu] edge node[auto=left,near start] {$1$} (one);
        \node(out) at (2,2) [linear] edge node {$2$} (h1) edge node[auto=left] {$-4$} (h2) edge node[auto=left] {$2$} (h3) edge node[auto=left] {$-1$} (h4);
      \end{tikzpicture}
      &
      \begin{tikzpicture}
        \begin{axis}[xmin=-2,xmax=2,ymin=-1.2,ymax=1.2,height=5cm,axis equal image,xlabel={$c_0 + \sum_i c_i x_i$},xtick={-2,-1,0,1,2},xticklabels={$-2$,$-1$,$0$,$1$,$2$},ytick={-1,0,1},yticklabels={$-1$,$0$,$1$}]
          \addplot[mark=none] coordinates { (-2,-1) (-1,-1) (0,1) (1,-1) (2,-1) };
        \end{axis}
      \end{tikzpicture}
    \end{tabular}
  \end{center}

  \paragraph{Case $\neg \chi_1$:} Since we are using a different representation of truth values, this is different from \cref{thm:lower_lemma1}.
  \begin{center}
    \begin{tikzpicture}
      \node(x) at (0.5,0) [input,label=below:$\chi_1$];
      \node(h1) at (0,1) [relu] edge node{$1$} (x);
      \node(h2) at (1,1) [relu] edge node[auto=left]{$-1$} (x);
      \node(out) at (0.5,2) [linear] edge node{$-1$} (h1) edge node[auto=left]{$1$} (h2);
    \end{tikzpicture}
  \end{center}
  
  \paragraph{Case $\chi_1 \land \chi_2$:}
  By the induction hypothesis, there are stacks of transformer layers of widths $d_1$ and $d_2$, respectively, that compute $\chi_1$ and $\chi_2$.
  Concatenate them by \cref{thm:concat}.
  Add a layer at the \emph{bottom} whose self-attention does nothing and whose FFN copies the input from the first half to the second half:

  \begin{center}
    \begin{tikzpicture}[y=1.5cm]
      \node(x11) at (0,0) [input,label=below:\strut$x_1$];
      \node at (1,0) {$\cdots$};
      \node(x1k) at (2,0) [input,label=below:\strut$x_k$];
      \node(u1) at (3,0) [input,label=below:\strut$1$];
      \node at (4,0) {$\cdots$};
      \node(x21) at (5,0) [input,label=below:\strut$0$];
      \node at (6,0) {$\cdots$};
      \node(x2k) at (7,0) [input,label=below:\strut$0$];
      \node(u2) at (8,0) [input,label=below:\strut$0$];
      \node at (9,0) {$\cdots$};
      \begin{scope}[every path/.style={thick,decorate,decoration={calligraphic brace,mirror}}]
      \draw (-0.2,-0.5) to node[auto=right] {$d_1$ units} (4.2,-0.5);
      \draw (4.8,-0.5) to node[auto=right] {$d_2$ units} (9.2,-0.5);
      \end{scope}
      
      \node(h11) at (0,1) [relu] edge (x11);
      \node at (1,1) {$\cdots$};
      \node(h1k) at (2,1) [relu] edge (x1k);
      \node(h1u) at (3,1) [relu] edge (u1);
      \node at (4,1) {$\cdots$};
      \node(h21) at (5,1) [relu] edge (x11);
      \node at (6,1) {$\cdots$};
      \node(h2k) at (7,1) [relu] edge (x1k);
      \node(h2u) at (8,1) [relu] edge (u1);
      \node at (9,1) {$\cdots$};
      
      \node(o11) at (0,2) [linear] edge (h11);
      \node at (1,2) {$\cdots$};
      \node(o1k) at (2,2) [linear] edge (h1k);
      \node(o1u) at (3,2) [linear] edge (h1u);
      \node at (4,2) {$\cdots$};
      \node(o21) at (5,2) [linear] edge (h21);
      \node at (6,2) {$\cdots$};
      \node(o2k) at (7,2) [linear] edge (h2k);
      \node(o2u) at (8,2) [linear] edge (h2u);
      \node at (9,2) {$\cdots$};
    \end{tikzpicture}
  \end{center}

  Finally, add a layer on top whose self-attention does nothing and whose FFN computes the minimum of the outputs of the two halves, as in the conjunction case of \cref{thm:lower_lemma1}.
  
  \paragraph{Case $\chi_1 \lor \chi_2$:} Same as the previous case, but compute the maximum instead of the minimum, as in the disjunction case of \cref{thm:lower_lemma1}.
\end{toappendix}

Using this lemma, we get a stack of transformer layers equivalent to $\chi$; call it $\Chi$.

Next, we want to compose $C$ and $\Chi$.
Let $d_C$ and $d_\Chi$ be the width of $C$ and $\Chi$, respectively, and let $d = \max \{d_C, d_\Chi\}$. If $d_C < d$, construct an encoder $Z$ that outputs $\mathbf{0}^{d-d_C}$ and let $C' = Z \oplus C$ and $\Chi'=\Chi$. Similarly if $d_\Chi < d$. Let $T = \Chi' \circ C'$; then $[T(w)]_{d,0}$ is positive iff the whole sentence is true.

Finally, the output layer \cref{eq:output_layer} projects $[T(w)]_{*,0}$ to dimension $d$, so the output probability is greater than $\tfrac12$ iff the whole sentence is true.

\subsection{Relationship to counter machines}
\label{sec:sscm}

\citet{bhattamishra+:2020} define a kind of counter machine called \emph{simplified stateless counter machine} (SSCM). It has zero or more counters, and upon reading each input symbol~$\letter$, it increments or decrements each counter by an amount that depends only on $\letter$. At the end of the input string, the accept/reject decision depends on whether the counters are zero.
Then they prove that any SSCM can be converted to an equivalent transformer. Since our definition of transformers and how they accept strings are slightly different from theirs, we give a slightly different definition of SSCM from theirs. We discuss these differences at the end of this section.

\begin{definition}
  A \emph{simplified stateless $k$-counter machine} \citep{merrill:2020,bhattamishra+:2020}, or $k$-SSCM, is a tuple $(\alphabet, u, F)$ where
  \begin{itemize}
  \item $\alphabet$ is a finite alphabet
  \item $u \colon \alphabet \rightarrow \mathbb{Z}^k$ is a counter update function
  \item $F \subseteq \{0, 1\}^k$ is an acceptance mask.
  \end{itemize}
\end{definition}

\begin{definition}
  Let $M$ be a $k$-SSCM, and let $w = w_1 \cdots w_n$ be an input string. We say that $M$ \emph{accepts} $w$ if there is a sequence $c_0, \dots, c_n \in \mathbb{Z}^k$ such that
  \begin{itemize}
  \item $c_0 = \mathbf{0}$
  \item $c_i = c_{i-1} + u(w_i)$ for all $i \in [1,n]$
  \item $[c_n]_i = 0$ iff $F_i = 0$.
  \end{itemize}
  We say that $M$ \emph{recognizes} a language $L$ if $L = \{ w \in \alphabet^* \mid \text{$M$ accepts $w$}\}$.
\end{definition}

\begin{proposition} \label{thm:counter}
  For any simplified stateless $k$-counter machine \citep{bhattamishra+:2020} there is a sentence of \logic{} that defines the same language.
\end{proposition}

\begin{proof}
  Let $\alphabet = \{\letter_1, \dots, \letter_m\}$, and
  let $M = (\alphabet, u, F)$ be a $k$-SSCM.
  Then the following sentence of \logic{} is equivalent to~$M$:
  \begin{align*}
    \sentence &= \exists x_1. \dots \exists x_m. \left( \bigwedge_{j=1}^m \counteq{x_j} p. \isletter{\letter_j}{p} \land \bigwedge_{i=1}^k \formula_i \right) \\
    \formula_i &= \begin{cases}
      u(\letter_1) x_1 + \dots + u(\letter_m) x_m = 0 & F_i = 0 \\
      u(\letter_1) x_1 + \dots + u(\letter_m) x_m \neq 0 & F_i = 1.
    \end{cases}
    \tag*{\qedhere}
  \end{align*}
\end{proof}

On the other hand, \logic{} can define languages that SSCMs cannot, making it a tighter lower bound than \citeauthor{bhattamishra+:2020}'s.
\begin{proposition} \label{thm:sscm}
  The language {\normalfont $(\texttt{01})^*$} is definable in \logic{} but cannot be recognized by any SSCM.
\end{proposition}
\begin{proof}
  The language is definable using the sentence \begin{align*}
    \exists x. &(\counteq{x}p. \isletter{\texttt{0}}{p} \land \counteq{x}p. \isletter{\texttt{1}}{p} \\ &\quad \land \forall p. \congruent{1}{2}{p} \leftrightarrow \isletter{\texttt{0}}{p}).
    \end{align*}

  To see that this language cannot be recognized by a SSCM, observe that SSCMs are permutation-invariant. That is, let $M$ be a SSCM. For any string $w = w_1 \cdots w_n$ and any permutation $\pi$ on $[1,n]$, define $\pi(w) = w_{\pi(1)} \cdots w_{\pi(n)}$. Then $M$ accepts $w$ if and only if it accepts $\pi(w)$. Since this means that $M$ cannot distinguish $(\texttt{01})^n$ from $\texttt{0}^n \texttt{1}^n$, it cannot recognize $(\texttt{01})^*$.
\end{proof}

\citeauthor{bhattamishra+:2020}'s definition of transformers differs from ours in two ways.
First, their transformer encoders have so-called causal (or future) masking, in which each position only attends to the position to the left. Ours do not, and indeed it appears that expressing causal masking would require a binary predicate $p < q$, which would break \cref{thm:normal_form}. We leave investigation of transformer encoders with causal masking for future work. Regardless, masking is not required for simulating SSCMs, as our \cref{thm:counter} holds without it.
  
Second, whereas our definition of transformer classifier places the output sigmoid layer over \cls, theirs places it over the last string position.
This is an arbitrary decision, but as a consequence, their definition of SSCM lets the accept/reject decision depend also on the last input symbol.

Consequently, \cref{thm:lower} does not, strictly speaking, improve on their lower bound.
We don't consider this to be a critical issue, however, because it doesn't seem to relate to essential properties of either transformers or counter machines, and it could easily be fixed, for example, by extending the transformer's positional encoding with a flag indicating the end of string and \logic{} with a predicate indicating the end of string.

\section{Related Work}
\label{sec:related}

Neural networks have been studied in relation to propositional logic from the start \citep{mcculloch+pitts:1943}. Much more recently, \citet{barcelo+:2020} relate graph neural networks to first-order logic with threshold counting quantifiers ($\exists^{\geq k}$ where $k$ is a constant) and at most two variables.

To our knowledge, the only attempts to relate transformers to formal logic are that of \citet{merrill+sabharwal:2023} and the present paper.
But there is a substantial literature on the expressivity of transformers, and in the rest of this section, we review some of this work, limiting our attention to results on transformers as recognizers of formal languages.

\subsection{Upper bounds}
\label{sec:related-upper}

Transformer encoders under various restrictions have been shown to fall into various language classes.
We have already discussed the upper bound of \citet{merrill+sabharwal:2023} using fixed-precision numbers in \cref{sec:tc0}, and review a few others here.

\Citet{hao-etal-2022-formal} study transformers with so-called \emph{hard attention},
where each position attends to the position with the highest attention logit. In the case of a tie, the leftmost position wins. They show (generalizing a result by \citet{hahn:2020} on $\textsf{PARITY}$ and the Dyck language with two pairs of brackets) that such transformers recognize languages in non-uniform $\textsf{AC}^0$ (that is, families of Boolean circuits with unlimited fan-in, polynomial size, and constant depth).

\Citet{merrill-etal-2022-saturated} study transformers with \emph{saturated attention} where, in the case of a tie, attention is distributed evenly among the tied positions. Additionally, they assume that all activations are numbers of the form $x/2^y$ where $x$ and $y$ are integers, with certain operations (reciprocal, square root) rounded to the nearest multiple of $1/2^y$. They show that such transformers recognize languages in non-uniform~$\textsf{TC}^0$.
In subsequent work \citep{merrill+sabharwal:2023log}, they show that transformers using full (softmax) attention and numbers with $O(\log n)$ bits recognize languages in logspace-uniform~$\textsf{TC}^0$.

Also worth mentioning is a result by \citet{hahn:2020}; he considers transformer classifiers whose activation functions are Lipschitz continuous (that is, if layer normalization is used, then $\epsilon>0$) and which always accept or reject strings with some (arbitrarily small, but fixed) margin. He shows that such transformer classifiers cannot recognize $\textsf{PARITY}$ or the Dyck language with two pairs of brackets.

All of the above upper bounds require some modification to the definition of transformer. Ours (\cref{thm:upper}) is no exception: although we use full (softmax) attention, we limit numbers to fixed-precision. Relaxing this restriction would break \cref{thm:fixed_function} and is left for future investigation.

\subsection{Lower bounds}
\label{sec:related-lower}

We have discussed the lower bound of \citet{bhattamishra+:2020} already in \cref{sec:sscm}.
The other lower bounds that we are aware of involve extensions of transformers. First, \citet{chiang+cholak:2022} showed that transformers whose PEs include a $p/n'$ component can recognize $\mathsf{PARITY}$.

Second, RASP \citep{weiss+:2021} is a programming language that can be compiled to transformers with saturated attention and several other extensions that appear to increase their expressivity. 
Their attention weights are directly computed from the previous layer, and are not restricted to be dot-products of query and key vectors; this allows compilation of expressions involving binary predicates like $p = q$ or $p < q$.
Their position-wise FFNs are allowed to compute arbitrary functions, the rationale being that they can be approximated by ReLU FFNs by the universal approximation theorem.

In contrast, our lower bound here follows the definition of a transformer encoder fairly strictly; the only departure is allowing the PE to have sine/cosine waves with different frequencies than the original definition.

Third, \citet{perez+:2021} consider transformers with saturated attention and several extensions. But more importantly, their result concerns, not encoders, but encoder--decoders. The encoder reads a string $w$, and the decoder is allowed to run for an arbitrary number of steps before making an accept/reject decision. This makes the model much more powerful: \Citet{perez+:2021} show that it can simulate a Turing machine.

The use of a decoder is of course standard and not an ``extension.'' But our result and theirs are in no way contradictory; they simply concern two very different configurations of transformers. Intuitively, one could liken a transformer encoder to a system that is prompted with a string and must immediately accept or reject, whereas a transformer encoder--decoder could be likened to a system that can ``think step by step'' \citep{kojima+:2022} before generating a final answer.

\section{Discussion}
\label{sec:discussion}

\subsection{Relationship with other complexity classes}

We have already discussed the relationship of \logic{} with $\mathsf{TC}^0$ (\cref{sec:tc0}) and SSCMs (\cref{sec:sscm}). We can further relate \logic{} to the classes of regular languages and uniform $\textsf{AC}^0$:

\begin{proposition} \label{thm:regular}
The class of languages recognizable by \logic{} is not comparable with the class of regular languages.
\end{proposition}
\begin{proof}
The language \textsf{MAJORITY}, containing those strings with more \texttt{1}'s than \texttt{0}'s, is definable in \logic{} 
by the sentence $\exists x. \exists y. \left( \counteq{x}p. \isletter{\texttt{0}}{p} \land \counteq{y}p. \isletter{\texttt{1}}{p} \land x > y \right)$, but is not regular.
On the other hand, $\texttt{0}^* \texttt{1}^*$ is regular but not definable in \logic, by an argument similar to \cref{thm:tc0}.
\end{proof}

\begin{proposition}
The class of languages recognizable by \logic{} is not comparable with uniform $\mathsf{AC}^0$.
\end{proposition}

\begin{proof}
The languages used in the proof of \cref{thm:regular} apply here as well: \textsf{MAJORITY} is not in $\textsf{AC}^0$ \citep{furst+:1984}, but $\texttt{0}^* \texttt{1}^*$ is in uniform $\textsf{AC}^0$, because $\textsf{AC}^0$ is equivalent to \textsf{FO+BIT}, which includes the sentence $\exists p. \forall q. (q < p \rightarrow \isletter{\texttt{0}}{p} \land q \ge p \rightarrow \isletter{\texttt{1}}{p})$.
\end{proof}

\subsection{Transformer variants}

\Cref{sec:related-upper} mentioned several restrictions of transformers, proposed to make finding upper bounds easier. We think these are interesting in their own right, and it would be worthwhile to clarify the relationships among them. One relationship is already implied by \cref{thm:upper,thm:lower}. The translation from fixed-precision transformers to \logic{} to arbitrary-precision transformers produces networks that only use uniform attention, which is a special case of saturated attention. So fixed-precision transformers are at most as powerful as saturated-attention transformers.

Similarly, \cref{sec:related-lower} mentioned several extensions of transformers, and curiously, all three of these previous lower bounds include $p$ or $p/n'$ in their PEs \citep[also cf.][]{yun+:2020}. Intuitively, this gives them the ability to translate between counts and positions, and we think this extension merits further study, both theoretical and experimental.

\subsection{Next steps}

Nonetheless, our ultimate goal is to exactly characterize the expressivity of unrestricted transformers with rational weights.
It might be thought that rational weights would add too much power, since they can store an unbounded amount of information; for example, an RNN with rational weights is equivalent to a Turing machine \citep{siegelmann+sontag:1995}. But this is true only if it is allowed to run for arbitrarily many time steps; if it runs for $n$ time steps, it is intermediate in power between real-time Turing machines and real-time RAM machines \citep{chen+:2017}. 
Since a transformer encoder only has fixed depth $L$, we think it is reasonable to hope for a logic that is both equivalent to rational-weighted transformers and has useful and interesting properties.

\logic{} is a significant step in that direction. By \cref{thm:upper}, we know that it can at least express anything that real-world transformer encoders can, and by \cref{thm:lower}, we know that it does not have any excess expressivity that does anything ``un-transformer-like.'' Crucial to these results is the normal form of \cref{thm:normal_form}.
Just as (with a fixed number of inputs) arbitrary Boolean functions can be expressed as a disjunction of conjunctions, or arbitrary continuous functions can be approximated by a FFN, in a similar way (with variable-length input strings), our normal form allows some fairly complicated properties to be expressed in a very simple form that can be mapped to transformers.
We speculate that a normal form result along the lines of \cref{thm:normal_form} will make an exact characterization of unrestricted transformer encoders possible, and that it will provide insight into how still more complex properties of strings can be computed by transformers.

\nosectionappendix 
\begin{toappendix}
\section{Layer Normalization}
\label{sec:layer_norm}

\emph{Layer normalization} shifts and scales a vector to have some learned mean and standard deviation \citep{ba+:2016}. In this appendix, we give a brief definition of layer normalization as it relates to our other definitions, and describe how to modify the proof of \cref{thm:upper,thm:lower} to take layer normalization into account.

\subsection{Definition}
\label{sec:layer_norm_def}

\begin{definition}
  \emph{Layer normalization} with width $d$ is a function
  \begin{align*}
    \layernorm \colon \R^{d} &\rightarrow \R^{d} \\
    \left[ \layernorm(x) \right]_i &= \gamma_i \frac{x_i - \bar x}{\sqrt{\text{Var}(x)+\epsilon}} + \beta_i \\
    \intertext{where}
    \bar x &= \frac1d \sum_i x_i \\
    \text{Var}(x) &= \frac1d \sum_i (x_i - \bar x)^2
  \end{align*}
  and $\gamma, \beta \in \R^d$ are learned.
  If $A \in \R^{d \times n'}$, we can write $\layernorm(A)$ by analogy with \cref{eq:vectorize}.
\end{definition}
The term $\epsilon$ is added in all implementations we are aware of for numerical stability, although the original definition \citep{ba+:2016} has $\epsilon=0$.

Then the equations for a transformer layer (\cref{def:layer}) are modified to:
\begin{align*}
    \layer \colon \R^{d \times n'} &\rightarrow \R^{d \times n'} \\
    A &\mapsto A'' \ \text{where} \\
    A' &= \layernorm^{(1)}\left(\sum_{h=1}^H \selfatt^{(h)}(A) + A\right) \\
    A'' &= \layernorm^{(2)}\left(\ffnn(A') + A'\right)
\end{align*}
  where $\layernorm^{(1)}$ and $\layernorm^{(2)}$ are layer normalizations.
Although some variants place layer normalization before each layer \citep{wang+:2019,nguyen+salazar:iwslt2019}, we follow the original definition, which places layer normalization after the residual connection.

\subsection{Modified proof of \cref{thm:upper}}
\label{sec:layer_norm_upper}

In \cref{sec:upper}, we justified the use of limited-precision, limited-range numbers by the fact that in a transformer whose activation functions are all Lipschitz continuous, the activations are bounded \citep{hahn:2020}. However, layer normalization with $\epsilon=0$ is not Lipschitz continuous.
Fortunately, we can show that activations are bounded even with $\epsilon=0$.
\begin{proposition}
  For any transformer encoder with layer normalization with $\epsilon = 0$, there exists $\intbits$ such that, for all~$\ell$, $i$, and~$p$, if $A^{(\ell)}_{ip}$ exists then $|A^{(\ell)}_{ip}| < 2^\intbits$.
\end{proposition}
\begin{proof}
  For any $x \in \R^d$, let $z_i = \frac{x_i - \bar{x}}{\sqrt{\text{Var}(x)}}$. Then
  \begin{align*}
    \frac1{z_i^2} &= \frac {\sum_{j=1}^d (x_j - \bar{x})^2}{d (x_i - \bar{x})^2} \geq \frac1d
  \end{align*}
  so $|z_i| \leq \sqrt{d}$, and layer normalization is bounded. The fact that all other sublayers are continuous implies that all activations are also bounded.
\end{proof}

As for the proof of \cref{thm:upper} itself, defining layer normalization in \logic{} is straightforward by \cref{thm:fixed_function}.

\subsection{Modified proof of \cref{thm:lower}}
\label{sec:layer_norm_lower}

In \cref{sec:lower}, we proved \cref{thm:lower} without layer normalization. Adding layer normalization complicates the construction somewhat.

In the first step (computing the $\psi_i$), we modify the encoding of truth values so that layer normalization has no effect. Instead of representing a truth value or a count using a single activation, we use a pair of activations (following \citet{chiang+cholak:2022}). In the first step (\cref{sec:psi}), we use the pair $(1, 0)$ for true and $(0, 1)$ for false. This makes it possible to guarantee that activation matrices have the following property.

\begin{definition}
  We say that a matrix $A \in \R^{d \times n'}$ has \emph{row-mean} $\mu$ and \emph{row-variance} $\sigma$ if $A_{:,p}$ has mean $\mu$ and variance $\sigma$ for all $p$. Then a function $f \colon \alphabet^* \rightarrow \R^{d \times n'}$ is \emph{self-normalizing} if $f(w)$ has row-mean and row-variance not depending on $w$, and a function $f \colon \R^{d \times n'} \rightarrow \R^{d \times n'}$ is self-normalizing if $f(A)$ has row-mean and row-variance depending only on $A$'s row-mean and row-variance.
\end{definition}

Self-normalization makes it possible to set the parameters of layer normalization so that it has no effect.
\begin{proposition}
  If $f \colon \R^{d \times n'} \rightarrow \R^{d \times n'}$ is self-normalizing then there exist $\beta, \gamma$ such that $\layernorm(f(A); \beta, \gamma) = f(A)$.
\end{proposition}

It is easy to modify the proof of \cref{thm:lower_lemma1} for the new representation of truth values and to guarantee that the resulting transformer encoder is self-normalizing. In particular, \cref{thm:concat} constructs a self-normalizing $T_1 \oplus T_2$, provided $T_1$ and $T_2$ are self-normalizing.

The second step (computing counting quantifiers) produces values of the form $\frac{x_i}{n'}$, which we now modify to produce pairs of values $(\frac{x_i}{n'}, -\frac{x_i}{n'})$ to guarantee that activation matrices have a row-mean of zero. We can no longer guarantee that activations matrices have known row-variance, so layer normalization will rescale activations.

Consequently, in the third step (\cref{sec:chi}), we use $(+\delta, -\delta)$ for true and $(-\delta, +\delta)$ for false, where $\delta$ can be any positive number. It is easy to show that \cref{thm:lower_lemma2} still holds.

\end{toappendix}

\section*{Acknowledgements}

We thank Brian DuSell and the anonymous reviewers for their comments. Peter Cholak was partially supported by NSF grant DMS-1854136, and Anand Pillay was supported by NSF grants DMS-1760212 and DMS-2054271.

\bibliography{transfol}

\end{document}